\pdfoutput=1
\documentclass[sigconf]{acmart}
\usepackage{booktabs} 
\usepackage{times}  
\usepackage{helvet}  
\usepackage{courier}  
\usepackage{url}  
\usepackage{graphicx}  
\usepackage{graphicx,color}
\usepackage{amsmath}
\usepackage{natbib}
\usepackage{algorithm}
\usepackage{algorithmic}
\usepackage{subcaption}
\usepackage{array}
\usepackage{enumitem,calc,flushend}

\usepackage{siunitx}
\usepackage{bm}

\newtheorem{proposition}{Proposition}

\newcommand{\hide}[1]{}
\newcommand{\he}[1]{{\textsf{\textcolor{red}{[From He: #1]}}}}
\newcommand{\lc}[1]{{\textsf{\textcolor{blue}{[From Lecheng: #1]}}}}
\newcommand{\yu}[1]{{\textsf{\textcolor{green}{[Yu: #1]}}}}
\newcommand{\reviewer}[1]{{\textsf{\textcolor{blue}{[From reviewer: #1]}}}}
\newcommand{\mkclean}{
  \renewcommand{\he}[1]{}
  \renewcommand{\lc}[1]{}
  \renewcommand{\reviewer}[1]{}
   \renewcommand{\yu}[1]{}
}
\mkclean

\ExplSyntaxOn
\newcommand\latinabbrev[1]{
  \peek_meaning:NTF . {
    #1\@}%
  { \peek_catcode:NTF a {
      #1.\@ }%
    {#1.\@}}}
\ExplSyntaxOff

\def\eg{\latinabbrev{e.g}}

\DeclareMathOperator{\softmax}{softmax}
\DeclareMathOperator{\sigmoid}{sigmoid}

\DeclareMathOperator{\E}{\mathbb{E}}

\newcommand{\method}{{\emph {ANTS}}}
\renewcommand\footnotetextcopyrightpermission[1]{} 

\def\BibTeX{{\rm B\kern-.05em{\sc i\kern-.025em b}\kern-.08emT\kern-.1667em\lower.7ex\hbox{E}\kern-.125emX}}
    
\setcopyright{acmcopyright}
\copyrightyear{2021}
\acmYear{2021}
\acmConference[WWW '21]{Proceedings of the Web Conference 2021}{April 19--23, 2021}{Ljubljana, Slovenia}
\acmBooktitle{Proceedings of the Web Conference 2021 (WWW '21), April 19--23, 2021, Ljubljana, Slovenia}
\acmPrice{}
\acmDOI{10.1145/3442381.3449801}
\acmISBN{978-1-4503-8312-7/21/04}

\begin{document}
\title{Deep Co-Attention Network for Multi-View Subspace Learning}
\author{Lecheng Zheng$^1$, ~~~Yu Cheng$^2$,~~~Hongxia Yang$^3$,~~Nan Cao$^4$ and Jingrui He$^1$}
\affiliation{
\institution{$^1$University of Illinois at Urbana-Champaign, \{lecheng4, jingrui\}@illinois.edu; \\
$^2$Microsoft AI, chengyu05@gmail.com;\\
$^3$Alibaba Group, firewater1984@gmail.com; \\
$^4$ Tongji University, nan.cao@gmail.com}
\country{
}}

\keywords{Multi-view Learning, Attention Mechanism, Interpretable Machine Learning}

\begin{abstract}
    Many real-world applications involve data from multiple modalities and thus exhibit the view heterogeneity. For example, user modeling on social media might leverage both the topology of the underlying social network and the content of the users' posts; in the medical domain, multiple views could be X-ray images taken at different poses. To date, various techniques have been proposed to achieve promising results, such as canonical correlation analysis based methods, etc. In the meanwhile, it is critical for decision-makers to be able to understand the prediction results from these methods. For example, given the diagnostic result that a model provided based on the X-ray images of a patient at different poses, the doctor needs to know why the model made such a prediction. However, state-of-the-art techniques usually suffer from the inability to utilize the complementary information of each view and to explain the predictions in an interpretable manner. 
    
    To address these issues, in this paper, we propose a deep co-attention network for multi-view subspace learning, which aims to extract both the common information and the complementary information in an adversarial setting and provide robust interpretations behind the prediction to the end-users via the co-attention mechanism. In particular, it uses a novel cross reconstruction loss and leverages the label information to guide the construction of the latent representation by incorporating the classifier into our model. This improves the quality of latent representation and accelerates the convergence speed. Finally, we develop an efficient iterative algorithm to find the optimal encoders and discriminator, which are evaluated extensively on synthetic and real-world data sets. We also conduct a case study to demonstrate how the proposed method robustly interprets the predictions on an image data set. 
\end{abstract}
\date{}
\maketitle
\section{Introduction}
In many real-world applications, data are usually collected from multiple sources or modalities, exhibiting the view heterogeneity. For example, many images posted on Facebook or Twitter are usually surrounded by the text descriptions, both of which can be considered as two distinct views; in the face attribute classification problem, the data consist of different poses of the same person, and each pose can be considered as a single view with complementary information to each other; in stock price forecasting, the related factors include not only historical stock prices and financial statements from companies, but also news, weather, etc.

Up to now, many researchers have proposed various techniques to model the view heterogeneity based on different assumptions.
Some works assume that there exists a latent lower-dimensional subspace shared by multiple views, and these views can be reconstructed from this subspace. However, these state-of-the-art methods usually suffer from the inability to utilize the complementary information of each view and the label information to enhance the quality of the representation. For example, canonical correlation analysis (CCA)~\cite{hotelling1936relations}, kernel canonical correlation analysis (KCCA)~\cite{SuGY17}, deep canonical correlation auto-encoder (DCCAE)\cite{WangLL16a} aim to explore the linear/non-linear transformation of multi-view features by maximizing the correlation coefficient of the two views. However, in addition to not utilizing the label information, the major drawback of \cite{wang2019deep, hotelling1936relations, SuGY17, WangLL16a} is that they ignore the complementary information in the data since the CCA based methods target to extract the common information that is mostly correlated between views, which might result in a sub-optimal solution.

Recent years have witnessed the tremendous efforts devoted to developing interpretable learning algorithms~\cite{guidotti2018survey, selvaraju2017grad, lundberg2017unified, ribeiro2016should, koh2017understanding,zhou2020domain}.
Understanding the reasons behind the prediction is of key importance for those who plan to take action based on the prediction. For example, in the medical domain, a doctor expects to know why the model made a prediction of a potential disease given the X-ray images of a patient taken at different poses; in the financial domain, when investors utilize a model to predict the trend of the stock price, they expect to see the reasons behind the prediction so that they could analyze a large number of assets to form a diversified portfolio and mitigate risks in a reasonable way. However, in multi-view learning, most existing techniques ignore the interpretability of the predictive model. Although we could simply apply the existing explanation methods on the concatenated multi-view features to interpret the result, it may suffer from the noisy features by accidentally including such features for interpretation. On the other hand, utilizing the consensus information of multi-view data could potentially help us build a more robust interpretable model against the noisy features.

Motivated by these limitations, we propose a deep \underline{a}dversarial co-atte\underline{nt}ion model for multi-view \underline{s}ubspace learning named \method, extracting both the shared information and the view-specific information in an adversarial manner and providing the robust interpretation behind the prediction to the end-users. Within this model, the co-attention encoder module is designed to extract the common information shared by multiple views with the help of a view discriminator, and interpret the predictive results by weighing the importance of the input features; the decoder module is designed to project the common information back to the original input space; the view discriminator is included to regularize the generation quality of the shared representation. After reconstructing the multi-view features based on the shared representation, we use the residual between the original input features and the reconstructed features as the view-specific information or the complementary information. We integrate both the common information and complementary information to yield richer and more discriminative feature representation. 
Our main contributions are summarized below:
\begin{itemize}
    \item A novel deep model for multi-view subspace learning, which extracts both the common information and the complementary information, and provides a robust model explanation for each view via co-attention module.
    \item Novel cross reconstruction loss and the use of label information to guide the construction of the latent representation.
    \item A case study that show how the proposed method interprets the predictions on an image data set.  
    \item Experimental results on synthetic and real-world data sets, which demonstrate the effectiveness of the proposed model.
\end{itemize}

The rest of this paper is organized as follows. After a brief review of the related work in Section 2, we introduce our proposed model for deep multi-view subspace learning in Section 3.  The systematic evaluation of the proposed method on synthetic and real-world data sets is presented in Section 4. In Section 5, we conduct a case study to show how the proposed method interprets the prediction on an image data set before we conclude the paper in Section 6.

\section{Related Work}
In this section, we briefly review the related work on multi-view learning, interpretable Learning, adversarial learning, and cycle-consistency.
\subsection{Multi-view Learning and Interpretable Learning}
Learning multi-view data has been studied for decades. In multi-view learning~\cite{abs-1304-5634}, researchers aim to model the similarity and difference among multiple views. The existing algorithms can be classified into three categories: 1) co-training based methods, 2) multiple kernel learning, and 3) subspace learning. Co-training~\cite{BlumM98} is one of the earliest methods proposed for multi-view learning, which aims to find the maximal consistency of several independent views given only a few labeled examples and many unlabeled ones. Since then, many variants of co-training have been proposed to find the consistency among views. For example, \cite{YuKRSR07} proposed an undirected graphical model for co-training to minimize the disagreement among multi-view classifiers. In multiple kernel learning, \cite{FarquharHMSS05} proposed a two-view Support Vector Machine method (SVM-2K) to find multiple kernels to maximize the correlation of the two views; \cite{sindhwani2005co} proposed a co-regularization method to jointly regularize two Reproducing Kernel Hilbert Spaces associated with the two views; \cite{andrew2013deep} proposed Deep Canonical Correlation Analysis to find two deep networks such that the output layers of the two networks are maximally correlated. As for the subspace learning, the authors of ~\cite{JiaoX17} proposed a deep multi-view robust representation learning algorithm based on auto-encoder to learn a shared representation from multi-view observations;\cite{HeDZYHL16} proposed online Bayesian subspace multi-view learning by modeling the variational approximate posterior inferred from the past samples;
\cite{DBLP:conf/icdm/ZhouH17} proposed M2VW for multi-view multi-worker learning problem by leveraging the structural information between multiple views and multiple workers; ~\cite{TianPZZM18} proposed CR-GAN method to learn a complete representation for multi-view generations in the adversarial setting by the collaboration of two learning pathways in a parameter-sharing manner. 
Different from~\cite{TianPZZM18, DBLP:conf/icdm/ZhouH17, DBLP:conf/ijcai/ZhouH16, DBLP:conf/sdm/ZhouYH17, DBLP:journals/tkdd/ZhouYH19, DBLP:conf/sdm/ZhengCH19, DBLP:conf/aaai/ZhouLSZHCK15, DBLP:conf/ijcai/ZhouHCD15, DBLP:conf/icdm/ZhouHCS16}, 
in this paper, we focus on multi-view classification problem and aim to extract both the shared information and the view-specific information in the adversarial setting, and the view consistency constraint with label information is utilized to further regularize the generated representation in order to improve the predictive performance. 

Recently, more and more studies on model explanation~\cite{guidotti2018survey, selvaraju2017grad, lundberg2017unified, ribeiro2016should, koh2017understanding,zhou2020domain} reveal a surge of research interest in the model interpretation. \cite{ribeiro2016should} is one of the earliest works in the model interpretation, which proposes the LIME algorithm to explain the predictions of any model in an interpretable manner. In~\cite{koh2017understanding}, the authors propose a black-box explanation algorithm to interpret how a training example influences the parameters of a model; in~\cite{zhou2020domain}, the authors propose a domain adaptive attention network to explore the relatedness of multiple tasks and leverage consistency of multi-modality financial data to predict stock price. In this paper, we leverage the co-attention mechanism to interpret the prediction by weighting the importance of the input features.

\subsection{Adversarial Learning and Cycle Consistency}
Adversarial learning is a technique attempting to fool models through malicious input, which is a promising way to train robust deep networks, and can generate complex samples across diverse domains \cite{GoodfellowPMXWOCB14,44904, NIPS2017_6815,zhu2020freelb, zhou2019misc, DBLP:conf/kdd/ZhouZ0H20}. The generative adversarial networks (GANs) \cite{GoodfellowPMXWOCB14} is an effective approach to estimate intractable probabilities, learned by playing a min-max game between generator and discriminator. 
\cite{44904} applied both reconstruction error and adversarial training criteria to a traditional auto-encoder. The MMD-GAN work \cite{NIPS2017_6815} introduced the adversarial kernel learning technique for the discriminator loss. 
Recently, some adversarial learning approaches \cite{ganin2016domain, DBLP:journals/tacl/ZhangBJ17,8099799, DuDXZW18, LiuQH17} have been proposed to minimize the distance between feature distributions. The domain-adaption works in \cite{ganin2016domain,DBLP:journals/tacl/ZhangBJ17,8099799} tried to learn a domain-invariant representation in adversarial settings. In~\cite{DuDXZW18}, the adversarially learned inference model aimed to find the shared latent representations of both views. \cite{LiuQH17} proposed an adversarial multi-task learning framework to separate the shared and private features of multiple tasks. Different from these methods solely focusing on aligning the global marginal distribution by fooling a domain discriminator, we explore to further align the learned representation by considering the label information. 
On the other hand, the cycle-consistency or the idea of using transitivity to regularize structured data has been applied in many applications, including image matching~\cite{ZhouLYE15}, co-segmentation~\cite{WangHOG14}, style transfer~\cite{ZhuPIE17}, etc. In~\cite{ZhouKAHE16, GodardAB17}, the cycle consistency constraint is utilized as a regularizer to push the mappings to be as consistent with each other as possible in the supervised convolution neural network training.
In~\cite{ZhuPIE17}, the authors proposed the Cycle-Consistent generative adversarial network (Cycle-GAN) to learn two mappings or generators and two discriminators by using transitivity to supervise CNN training.
Different from Cycle-GAN~\cite{ZhuPIE17}, our paper mainly focuses on finding a latent representation shared by multiple views by leveraging the label information and extending the cycle consistency idea, rather than transforming data from one domain to another domain.

\section{Proposed \method\ Framework}
In this section, we present our proposed \method, a deep co-attention multi-view subspace learning framework. We start by introducing the major notation used in this paper and then discuss the extraction of both the shared latent information and the view-specific information, as well as how to leverage the label information and incorporate the classification model. Notice that in the subsequent discussion, for the sake of exposition, we focus on multi-view learning with two views. The extension of the proposed techniques to the more generic settings with multiple views is discussed in Subsection 3.4.

\begin{figure}[ht]
\includegraphics[width=1\linewidth]{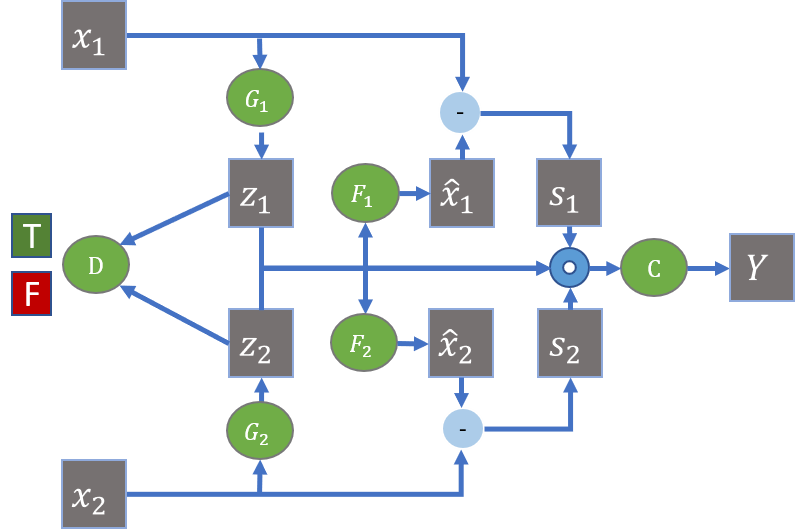}
\caption{Given an example drawn from a data set with two views $\bm{x_1}$ and $\bm{x_2}$,  we use two encoders $G_1(\cdot)$ and $G_2(\cdot)$ to generate the shared latent representation $\bm{z}$. The goal of the discriminator $D(\cdot)$ is to determine whether the shared latent representation $\bm{z}$ is generated by the first view (True) or the second view (False). To constrain the generated representation, two decoders $F_1(\cdot)$ and $F_2(\cdot)$ are utilized to transform the shared representation (either $z_1$ or $z_2$) back to two recovered views $\hat{\bm{x}}_1$ and $\hat{\bm{x}}_2$, and we want the recovered views to be as similar to the real views as possible.
\he{There's no discriminator trying to distinguish the original features and the reconstructed features. So you should not say that the reconstructed views are fake, as they were not generated from random noise, but rather the shared representation obtained from the encoder.} \lc{updated}
Meanwhile, we extract the view-specific information $s_1$ and $s_2$ by simply subtracting $\hat{\bm{x}}_1$ from $\bm{x_1}$ and subtracting $\hat{\bm{x}}_2$ from $\bm{x_2}$, respectively. After extracting both the shared information and the view-specific information, we concatenate them in a hidden layer and feed them into a classifier $C(\cdot)$. Finally, label information $\bm{Y}$ is utilized to enforce the encoders to create better representations in order to improve the predictive performance of the classifier.}
\he{From this figure, I do not see any difference between your way of utilizing the label information and the traditional way of constructing a classifier.}\lc{updated}
\label{structure}
\end{figure}

\subsection{Notation}
Throughout this paper, we use lower-case letters for scalars (e.g., $\gamma$), a bold lower-case letter for a sample (e.g., $\bm{x^j}$), and a bold upper-case letter for a matrix (e.g., $\bm{X}$).
We use $\mathcal{D} ={(\bm{X_1}, \ldots, \bm{X_v}, \bm{Y})}$ to denote a data set, where $\bm{X_i} \in \mathbb{R}^{n \times l_i}$ is the feature for the $i^{th}$ view, $\bm{Y}\in \mathbb{R}^{n \times c}$ is the binary label matrix, $n$ is the number of samples, $l_i$ is the dimension of the input feature, $v$ is the number of views which is set to be 2 if not specified, and $c$ is the number of classes. For a single sample $\bm{X_i(j,:)}\in \mathbb{R}^{l_i}$, if it is image data, we could apply image partition algorithm, \eg,~Mask R-CNN~\cite{he2017mask}, to split this image into multiple segments; if it is text data, we could use Word2vec~\cite{mikolov2013distributed} to extract the input feature. Following this idea, we could reshape the dimension of a sample, \eg., $\bm{x_i^j} \in \mathbb{R}^ {d_i \times k_i}$, where $k_i \times d_i = l_i$, $\bm{x_i^j}=\{\bm{x_{i,1}^j}; \bm{x_{i,2}^j}; ...; \bm{x_{i, k_i}^j}\}$ and $\bm{x_{i, k_i}^j}\in \mathbb{R}^{d_i}$. For image data, $k_i$ is the number of image segments and $d_i$ is the dimension of the representation of each segment (for text data, $k_i$ is the number of words and $d_i$ is the dimension of word embedding).
We denote $G_i(\cdot)$ and $F_i(\cdot)$ as an encoder-decoder pair for the $i^{th}$ view, one for projecting the data point to the shared subspace and another for transferring the projection back to the original data point. We also denote $D(\cdot)$ as the discriminator, $C(\cdot)$ as the classifier and $P_{g_i}(\bm{X_i})$ as a prior on the $i^{th}$ view. 
Given a sample $x^j$, we denote $\bm{z}\in \mathbb{R}^{h}$ as the shared representation generated by encoders, $\hat{\bm{x}}_i^j \in \mathbb{R}^{d_i \times k_i}$ as the $i^{th}$ the recovered view generated by $F_i(\cdot)$, $\bm{z^j}\in \mathbb{R}^{h}$ as the representation generated from a sample $x_i^j$ by $G_i(\cdot)$, where $h$ is the dimension of the hidden representation. Note that $\hat{Y}\in \mathbb{R}^{n \times c}$ is the prediction made by the classifier $C(\cdot)$ for labeled data, and $\hat{Y}_i\in \mathbb{R}^{n \times c}$ is the prediction of the $i^{th}$ view for unlabeled data.

\subsection{Objective Function}
Now, we are ready to introduce the overall objective function:
\begin{equation}
    \label{overall}
    \begin{split}
        \min_{G_1, G_2, F_1, F_2, C} \max_{D} L &= L_0(G_1, G_2, D)  + \alpha L_v(G_1, G_2, F_1, F_2)\\
        & + \beta L_c(C, G_1, G_2, F_1, F_2)
    \end{split}
\end{equation}
where $\alpha$ and $\beta$ are two positive constants that balance the three terms in the objective function: $L_0$ is the objective function of the min-max game, $L_v$ is the cross reconstruction loss to regularize the generation of the representation and $L_c$ is the cross-entropy loss by leveraging the label information. Figure~\ref{structure} provides an overview of the proposed framework. Next, we elaborate on each module respectively.

\begin{figure}[ht]
\includegraphics[width=0.80\linewidth]{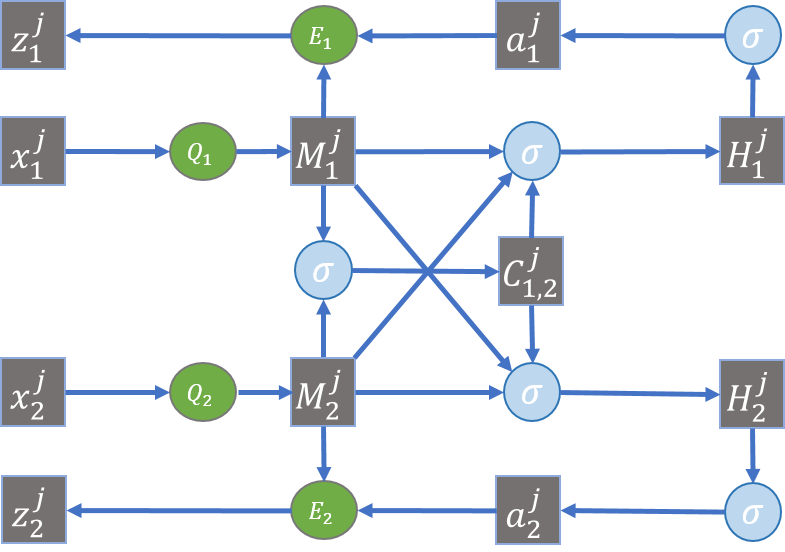}
\caption{Co-attention encoder module, where $Q_1$, $Q_2$, $E_1$ and $E_2$ are neural networks, and $\sigma$ is the non-linear transformation function, \eg, $\tanh$ or $\softmax$.}
\label{co-attention encoder}
\end{figure}

\subsubsection{Co-attention Encoder Module}
The main idea of co-attention encoder is to explore the correlation of each pair of features in two views
in order to select the most important features by weighing the importance of the input features via co-attention mechanism~\cite{lu2016hierarchical, liu2020decoupled}. Please refer to Figure~\ref{co-attention encoder} for the network structure of the co-attention encoder module. Suppose we are given a sample $\bm{x^j}$ with two views, the affinity matrix $\bm{C_{1,2}^j}\in \mathbb{R}^{k_1 \times k_2}$ of these two views is computed by:
\begin{equation}
    \begin{split}
        \bm{M_1^j} &= Q_1(\bm{x_1}^j), \indent \bm{M_2^j} = Q_2(\bm{x_2}^j) \\
        \bm{C_{1,2}^j} & = \tanh{(\bm{(M_1^j)}^T  \bm{W_{1,2}}\bm{M_2^j})}
    \end{split}
\end{equation}
where $\tanh(x)=\frac{e^x-e^{-x}}{e^x+e^{-x}}$,  $\bm{M_1}\in \mathbb{R}^{d_1 \times k_1}, \bm{M_2}\in \mathbb{R}^{d_2 \times k_2}$ are two hidden representations mapped by two neural networks ($\textit{Q}_1$ and $\textit{Q}_2$) and $\bm{W_{1,2}}\in \mathbb{R}^{d_1 \times d_2}$ is a weight matrix shared by all samples. 
Basically, this affinity matrix measures the similarities between the segments among two views. By leveraging such an affinity matrix of two views, we aim to encode the information from both views in the hidden representation $H_1$ and $H_2$ that could be formulated as follows:
\begin{equation}
    \begin{split}
        H_1^j &= \tanh{(\bm{W_1}\bm{M_1^j} + (\bm{W_2} \bm{M_2^j})\bm{(C_{1,2}^j)^T)}} \\
        H_2^j &= \tanh{(\bm{W_2}\bm{M_2^j} + (\bm{W_1}\bm{M_1}^j)\bm{C_{1,2}^j)}}
    \end{split}
\end{equation}
where $\bm{W_1}\in \mathbb{R}^{d_3 \times d_1}$ and $\bm{W_2}\in \mathbb{R}^{d_3 \times d_2}$ are weight matrices and $d_3$ is the output dimension. By encoding the information from both views into two hidden representations $H_1^j$ and $H_2^j$, we aim to capture the consensus information of two views so that we could build a more robust interpretable model to mitigate the negative impact of noisy features. The importance of the input features is measured as follows:
\begin{equation}
\label{eq:attention_weights}
    \begin{split}
        \bm{a_1^j} = \softmax{}{(\bm{w_{h1}}\bm{H_1^j})}, \indent \bm{a_2^j} = \softmax{(\bm{w_{h2}}\bm{H_2^j})}
    \end{split}
\end{equation}
where $\bm{w_{h1}}\in \mathbb{R}^{d_3}$ and $\bm{w_{h2}}\in \mathbb{R}^{d_3}$ are weight vectors. Based on the above attention weight vectors $\bm{a_1^j} \in \mathbb{R}^{k_1}$ and $\bm{a_2^j} \in \mathbb{R}^{k_2}$, the extracted representation for the first view and the second view can be calculated by:
\begin{equation}
\label{attention}
    \begin{split}
        \bm{z_1^j} = E_1(\sum_{l=1}^{k_1} \bm{a_{1, l}^j} \bm{M_{1, l}^j}), \indent \bm{z_2^j} = E_2(\sum_{l=1}^{k_2} \bm{a_{2, l}^j} \bm{M_{2, l}^j})
    \end{split}
\end{equation}
where $E_1(\cdot)$ and $E_2(\cdot)$ are two neural networks that map two views to a latent space denoted as $\bm{z}$. Here, we denote the entire structure of the co-attention module as $G_1(\cdot)$ and $G_2(\cdot)$ for the first and second view, respectively (\eg.,~$z_i = G_i(x_i)$). 
In the next subsection, we introduce the discriminator module to help encoders extract the common representation shared by two views.

\subsubsection{Discriminator Module for Extraction of Shared Representation}
The discriminator module aims to help the encoder module find the latent representation shared by two views, which will be utilized for training a high-quality classifier. To learn the distribution of the two encoders $P_{g_1}$ and $P_{g_2}$ over the shared latent space $\bm{\bm{z}}$, we define $P_{g_1}(\bm{X_1})$ to be a prior on the first view $\bm{X_1}$ and $P_{g_2}(\bm{X_2})$ to be a prior on the second view $\bm{X_2}$. Thus, $\bm{\bm{z}} \sim P_{g_1}(\bm{X_1})$ is produced by the first data encoder $G_1(\cdot)$ fed with view $\bm{X_1}$; at the same time, $\bm{\bm{z}} \sim P_{g_2}(\bm{X_2})$ is produced by the second data encoder $G_2(\cdot)$ fed with view $\bm{X_2}$. Let $D(\cdot)$ be the discriminator that distinguishes whether the shared latent representation $\bm{z}$ is generated by the first encoder $G_1(\bm{x_1})$ or by the second encoder $G_2(\bm{x_2})$. The goal of the discriminator $D(\cdot)$ is to maximize the probability of assigning the correct labels to the representation generated by both encoders, e.g., assigning true to the representation generated by the first view and false to the representation generated by the second view. Two encoders aim to minimize the probability that the discriminator $D(\cdot)$ successfully distinguishes the representation generated from the first view or the second view.  Thus, the min-max objective function can be formulated by:
\begin{equation}
    \label{disc}
    \begin{split}
        \min_{G_1, G_2}\max_{D} L_0  & = \E_{\bm{x_1}\sim p(\bm{X_1})}[\log(D(G_1(\bm{x_1})))] \\
        & + \E_{\bm{x_2}\sim p(\bm{X_2})}[\log(1 - D(G_2(\bm{x_2})))]\\
        & = \E_{\bm{\bm{z}}\sim P_{g_1}(\bm{X_1})}[\log(D(\bm{\bm{z}}))] \\
        & + \E_{\bm{\bm{z}}\sim P_{g_2}(\bm{X_2})}[\log(1 - D(\bm{\bm{z}}))]
    \end{split}
\end{equation}
\he{The notation in the above equation is not correct.} \lc{updated}
In this min-max game, two encoders $G_1(\cdot)$ and $G_2(\cdot)$ aim to generate the same latent representation $\bm{\bm{z}}$ shared by both views so that the discriminator $D(\cdot)$ cannot reliably decide which encoder creates such a latent representation. 

\begin{proposition}
For the fixed encoders $G_1$ and $G_2$, the optimal discriminator $D^*(\bm{z})$ of Equation ~\ref{overall} is given by $D^*(\bm{z}) = \frac{P_{g_1}}{P_{g_1} + P_{g_2}}$.
\end{proposition}
\begin{proof}
Given the fixed encoder $G_1$ and $G_2$, the training criterion for the discriminators is to maximize\\
$L_0 = \int_z P_{g_1}\log(D(\bm{z})) dz + \int_z  P_{g_2} \log(1 - D(\bm{z}))dz$\\
$L_0  = \int_z [P_{g_1}\log(D(\bm{z})) + P_{g_2} \log(1 - D(\bm{z}))]dz$\\
For any (a,b) $\in \mathbb{R}^2 \setminus {(0,0)}$, the function $f(c)= a\log(c) + b(\log(1-c))$ achieves its maximum at $c=\frac{a}{a+b}$. Thus, $D^*(\bm{z}) = \frac{P_{g_1}}{P_{g_1} + P_{g_2}}$. 
\end{proof}

\subsubsection{Decoder Module for Extraction of View-specific Representation}
In this subsection, we propose to reconstruct the views based on the extracted shared representation to enforce the shared representation to be robust and only contain the common information. Besides, subtracting the reconstructed data from the original input features leaves us with the information specific to each view, which could be utilized to further boost the predictive performance. To reconstruct the views, we exploit the idea of transitivity to regularize structured data used in~\cite{ZhuPIE17} in order to restrict the distribution of the representation.
Our main intuition is that we could find the latent representation shared by the two views in a low-dimensional space and reconstruct these two views from this shared representation generated by either view. In other words, since the representation $\bm{z}$ is shared by two views, if we use the encoder $G_2(\bm{x_2})$ to generate the shared representation $\bm{z}$, the decoder $F_1(\bm{z})$ should be able to transform the representation $\bm{z}$ back to the first view $\bm{x_1}$. Similarly, the decoder $F_2(\bm{z})$ should also be able to transform the representation $\bm{z}$ generated by the encoder $G_1(\bm{x_1})$ back to the second view $\bm{x_2}$. 
Therefore, we let $\hat{\bm{x}}_i=\gamma_1 \hat{\bm{x}}_{(i,1)} + \gamma_2 \hat{\bm{x}}_{(i,2)}$, where $\gamma_1 + \gamma_2=1$,  $\hat{\bm{x}}_{(i,j)}$ means the reconstruction of the $i^{th}$ view from the $j^{th}$ view, and $\gamma_1, \gamma_2\in [0,1]$ are positive learnable parameters that balance between the two reconstructed views. In this way, the loss function can be updated as follows:
\he{It should be $L_v$ instead of $L$, right?} \lc{updated}
\begin{align}
        L_v & = \E_{\bm{x_1}\sim p(\bm{X_1})} \gamma_1 \|\bm{x_1}-\hat{\bm{x}}_1 \|_2^2 + \E_{\bm{x_2}\sim p(\bm{X_2})} \gamma_2 \|\bm{x_2}-\hat{\bm{x}}_2 \|_2^2 \notag\\
                 & = \E_{\bm{x_1}\sim p(\bm{X_1})} \gamma_1 \|\gamma_1\bm{x_1} + \gamma_2\bm{x_1}- \gamma_1\hat{\bm{x}}_{(1,1)}-\gamma_2\hat{\bm{x}}_{(1,2)}\|_2^2 \notag\\
                 & + \E_{\bm{x_2}\sim p(\bm{X_2})} \gamma_2 \|\gamma_1\bm{x_2} +\gamma_2 \bm{x_2}- \gamma_2\hat{\bm{x}}_{(2,2)}-\gamma_1\hat{\bm{x}}_{(2,1)} \|_2^2 \notag\\
                 & = \E_{(\bm{x_1}, \bm{x_2})\sim p(\bm{X_1}, \bm{X_2})}  \gamma_1 \|\gamma_1(\bm{x_1}- \hat{\bm{x}}_{(1,1)}) + \gamma_2 (\bm{x_1}-\hat{\bm{x}}_{(1,2)})  \|_2^2 \notag \\
                 & + \E_{(\bm{x_1}, \bm{x_2})\sim p(\bm{X_1}, \bm{X_2})}  \gamma_2 \|\gamma_2(\bm{x_2}- \hat{\bm{x}}_{(2,2)}) + \gamma_1(\bm{x_2}- \hat{\bm{x}}_{(2,1)}) \|_2^2 \notag\\
                 & \leq 2\E_{(\bm{x_1}, \bm{x_2})\sim p(\bm{X_1}, \bm{X_2})}[\gamma_1^3\|\bm{x_1}- \hat{\bm{x}}_{(1,1)}\|_2^2 +  \gamma_2^3\|\bm{x_2}- \hat{\bm{x}}_{(2,2)}\|_2^2\notag\\
                 & + \gamma_1 \gamma_2^2 \|\bm{x_1}- \hat{\bm{x}}_{(1,2)}\|_2^2 + \gamma_2\gamma_1^2 \|\bm{x_2}- \hat{\bm{x}}_{(2,1)}\|_2^2]\notag
\end{align}
where the last inequality is based on $\|a+b\|_2^2\leq 2\|a\|_2^2 + 2\|b\|_2^2$, $\forall a,b\in\mathbb{R}$. 
Based on the above analysis, we define the view reconstruction loss $L_v$ in the overall objective function as follows:
\he{Remove $\min$ in the equation below.}\lc{updated}
\begin{align}
\label{loss_3}
        L_v & = \E_{(\bm{x_1}, \bm{x_2})\sim p(\bm{X_1}, \bm{X_2})}[\gamma_1^3\|\bm{x_1}- \hat{\bm{x}}_{(1,1)}\|_2^2 +  \gamma_2^3\|\bm{x_2}- \hat{\bm{x}}_{(2,2)}\|_2^2\notag\\
                 & + \gamma_1 \gamma_2^2 \|\bm{x_1}- \hat{\bm{x}}_{(1,2)}\|_2^2 + \gamma_2\gamma_1^2 \|\bm{x_2}- \hat{\bm{x}}_{(2,1)}\|_2^2]
\end{align}
One advantage of this constraint is that it enforces the representations $\bf{z}_1$ and $\bf{z}_2$\hide{\he{Is it $\bf{z}_1$ and $\bf{z}_2$ instead?}} generated by two encoders $G_1(\cdot)$ and $G_2(\cdot)$ to be as similar as possible in order to minimize the reconstruction error, which helps the encoders to find the shared representation at a faster pace. After extracting the common information shared by multiple views, we also want to extract the view-specific representation to further boost the predictive performance because the view-specific information contains complementary features for the sake of classification. 
For each sample $x_i$, the view-specific information $\bm{s_i}$ for the $i^{\textrm{th}}$ view could be easily extracted by subtracting the constructed view from the original view, written as follows:\\
\begin{equation}
    \label{residual}
    \begin{split}
    \bm{s_i} &= \bm{x_i}-\hat{\bm{x}}_i = \bm{x_i}-\gamma_1 \hat{\bm{x}}_{(i,1)} - \gamma_2 \hat{\bm{x}}_{(i,2)} \\
    \end{split}
\end{equation}
To leverage both the shared and complementary information from all the views for the downstream classification task, we propose to concatenate both common information and complementary information from all views to make final prediction.

\subsubsection{Classification Module for Incorporating Label Information}
With both the shared information and the view-specific information, we propose to incorporate the classifier in our model to enhance the quality of the concatenated representation.
As mentioned earlier, one drawback of CCA based methods is that they only capture the information learned from the features but ignore the useful information from the labels. Nevertheless, the label information can be used to significantly improve the model performance, because the examples with different labels tend to have different features, and utilizing the label information may result in better representation for classification purposes. In our setting, we denote $C(\cdot)$ as a multi-layer neural network or a classifier that takes both the shared representation and the view-specific information as the input and outputs the prediction results.
We denote $\bm{Y}$ as the ground truth label and $\hat{Y}=C(z_s \oplus s_1 \oplus s_2)$ as the label predicted by the classifier, where $\oplus$ is concatenation operator and $z_s= (z_1 + z_2)/2$. Here we aim to minimize the following loss function:
\he{Please remove $\min$ below as this is only one part of your overall objective.}\lc{updated}
\begin{equation}
    \label{clf}
    \begin{split}
        L_c & =  \E_{(\bm{x_1}, \bm{x_2})\sim p(\bm{X_1}, \bm{X_2})}\mathbb{H}(\bm{\bm{Y}}, \hat{Y})
    \end{split}
\end{equation}
where $\mathbb{H}(\bm{\bm{Y}}, \hat{Y})$ is the cross-entropy loss.
One goal of the incorporated classifier is to help the two encoders $G_1(\cdot)$ and $G_2(\cdot)$ find a better representation by leveraging the label information. Notice that according to Eq.~\ref{attention}, with the help of the incorporated classifier, \method~ tends to assign a small weight to noisy or irrelevant features but a large weight to important features to minimize the prediction loss. This could further help us get a robust interpretable model and alleviate the negative influence of the noisy features. Meanwhile, when the final shared representation is determined, the classifier is also well-trained and ready for the prediction, which saves the time for training the classifier after finding the shared representation.

In addition, it is straightforward to extend the model to include unlabeled data. Similar to the idea of co-training~\cite{BlumM98}, these two predicted labels for unlabeled data should be consistent with each other. Thus, the loss function can be revised as follows:
\begin{equation}
    \begin{split}
        L_c = \E_{(\bm{x_1}, \bm{x_2})\sim p(\bm{X_1}, \bm{X_2})}\mathbb{H}(\hat{Y}_1, \hat{Y}_2) + \mathbb{H}(\bm{\bm{Y}}, \hat{Y}_1) + \mathbb{H}(\bm{\bm{Y}}, \hat{Y}_2)
    \end{split}
\end{equation}
where $\hat{Y}_1=C_1(z_1 \oplus s_1)$, $\hat{Y}_2=C_2( z_2 \oplus s_2)$ are the predictions of the first view and the second view for unlabeled data,
respectively and $C_1(\cdot)$ and $C_2(\cdot)$ are multi-layer neural network that take the concatenation of both view-specific information and shared representation as the input, and output the prediction results. The first term aims to minimize the inconsistency between two predicted labels for unlabeled data, and the remaining terms are the cross-entropy loss on the labeled data for the two views respectively.

\he{For the following algorithm, change Require to Input, and Ensure to Output.}\lc{updated}
\begin{algorithm}[t]
    \begin{algorithmic}
        \caption{\method\ Algorithm}
        \label{alg1}
        \REQUIRE  The total number of iterations $T$, two views $\bm{X_1}$ and $\bm{X_2}$, and the label $\bm{\bm{Y}}$, parameters $t_1$, $t_2$, $t_3$, $\alpha, \beta$.\\
        \ENSURE The well-trained classifier $C$.\\
        \STATE Randomly initialize the weights of the neural network.
        \FOR{$t=1$ to $T$}
            \STATE \textbf{Step 1: } Update the discriminator $D$ for $t_1$ times via stochastic gradient ascent by fixing the rest variables.\\
            $D \sim \nabla \E_{\bm{x_1, x_2}\sim p(\bm{X_1, X_2})}[\log(D(G_1(\bm{x_1}))) + \log(1 - D(G_2(\bm{x_2})))]$\\
            \STATE \textbf{Step 2: } Update $G_1$, $G_2$, $F_1$ and $F_2$ for $t_2$ times via stochastic gradient descent by fixing the rest variables.\\
            $G_1 \sim \nabla \E_{\bm{x_1, x_2}\sim p(\bm{X_1, X_2})} [\log D(G_1(\bm{x_1})) + \beta \|\bm{Y}-\hat{Y}\|_2^2$ \\
            \quad\quad\quad $ + \alpha   \gamma_2 \gamma_1^2 \|\bm{x_2}-F_2(G_1(\bm{x_1})) \|_2^2 + \alpha \gamma_1^3\|\bm{x_1}-F_1(G_1(\bm{x_1})) \|_2^2]$\\
            $G_2 \sim \nabla \E_{\bm{x_1, x_2}\sim p(\bm{X_1, X_2})}[\log(1-D(G_2(\bm{x_2}))) + \beta \|\bm{Y}-\hat{Y}\|_2^2$\\
            \quad\quad\quad $+ \alpha  \gamma_1 \gamma_2^2 \|\bm{x_1}-F_1(G_2(\bm{x_2}))\|_2^2 + \alpha \gamma_2^3 \|\bm{x_2}-F_2(G_2(\bm{x_2}))\|_2^2 ]$\\
            $F_1 \sim \nabla \E_{\bm{x_1, x_2}\sim p(\bm{X_1, X_2})} [\alpha  \gamma_1 \gamma_2^2\|\bm{x_1}-F_1(G_2(\bm{x_2})) \|_2^2 + \beta \|\bm{Y}-\hat{Y}\|_2^2$\\
            \quad\quad\quad $ + \alpha \gamma_1^3 \|\bm{x_1}-F_1(G_1(\bm{x_1})) \|_2^2]$\\
            $F_2 \sim \nabla \E_{\bm{x_1, x_2}\sim p(\bm{X_1, X_2})} [\alpha  \gamma_2^3\|\bm{x_2}-F_2(G_2(\bm{x_2})) \|_2^2 + \beta \|\bm{Y}-\hat{Y}\|_2^2$\\
            \quad\quad\quad $ + \alpha  \gamma_2 \gamma_1^2\|\bm{x_2}-F_2(G_1(\bm{x_1})) \|_2^2]$\\
            \STATE \textbf{Step 3: } Update the classifier $C$ for $t_3$ times via stochastic gradient descent by fixing the rest variables.\\
            $C \sim \nabla \E_{\bm{x_1, x_2}\sim p(\bm{X_1, X_2})} \|\bm{\bm{Y}}-C(z_s \oplus S_1 \oplus S_2)\|_2^2$ 
        \ENDFOR
    \end{algorithmic}
\end{algorithm}

\subsection{Proposed Algorithm}
We propose to solve the overall objective function via the stochastic block coordinate gradient descent method. Our algorithm is presented in Algorithm \ref{alg1}. It takes as input a data set with two views and the label information, the total number of iterations, as well as several parameters, and outputs the well-trained model. The algorithm works as follows. In Step 1, we first update the discriminator for $t_1$ times by fixing the other variables and performing a stochastic gradient ascent.
Then in Step 2, we only update the two encoders and the two decoders for $t_2$ times via stochastic gradient descent, as discussed in the previous subsection regarding the optimal encoders and decoders. Finally, in Step 3, we update the classifier for $t_3$ times via stochastic gradient descent.

\subsection{Extension to Multiple Views}
In this subsection, we extend our proposed techniques to more generic settings with multiple views. For each pair of views $x_i^j$ and $x_k^j$, we could compute its affinity matrix $C_{i,k}^j$ and the hidden representation $H_i$ as follows:
\begin{equation}
    \begin{split}
        \bm{C_{i,k}^j} & = \tanh{((\bm{M_i^j})^T \bm{W_{i,k}}\bm{M_k^j})} \\
         H_i^j &= \tanh{(\bm{W_i}\bm{M_i^j} + \sum_{k\neq i}^v(\bm{W_k} \bm{M_k^j})\bm{(C_{i,k}^j)^T)}} 
    \end{split}
\end{equation}
where $\bm{W_{i,k}}\in \mathbb{R}^{d_i \times d_j}$, $\bm{W_i}\in \mathbb{R}^{d \times d_i}$ and $\bm{W_k}\in \mathbb{R}^{d \times d_k}$ are weight matrices, and $d$ is the output dimension. Using the one vs. all strategy~\cite{RifkinK03}, we propose \textit{Centroid-}\method~ framework for the scenario with multiple views, where we consider the first view as the centroid view (assuming we have the prior knowledge that the first view is the most important view). With this setting, we want the discriminator $D(\cdot)$ to distinguish whether the representation is generated by the first view $\bm{X_1}$ or not, instead of determining whether the representation is generated by the first view or the second view. Thus, the min-max objective function could be adjusted as follows:
\begin{equation}
    \begin{split}
        L_0^{(v)} & = \E_{\bm{x}\sim p(\bm{X_1})}\log(D(G_1(\bm{x_1}))) \\ 
        & + \frac{1}{v-1}\sum_{i=2}^v \E_{\bm{x}\sim p(\bm{X_i})} \log(1 - D(G_i(\bm{x_i}))) \\
    \end{split}
\end{equation}
The goal of the discriminator $D(\cdot)$ is to maximize the probability of labeling true to the representation generated by the first view and labeling false to the representation generated by the other views. This enforces the representation generated by the other views to be as similar to the representation generated by the first view as possible.
Similar to the cross reconstruction loss introduced in the previous subsection, we expect that for each view other than the first view, the decoder $F_i(\cdot)$ can transform the representation generated by the $i^{\textrm{th}}$ view back to the first view. In addition, we also expect that $F_i(\cdot)$ be able to transform the representation back to the $i^{\textrm{th}}$ view. Thus, the loss function could be modified as follows:
\begin{equation}
    \label{loss_3_multiple_views}
    \begin{split}
        L_v^{(v)} & = \E_{\bm{x_1}\sim p(\bm{X_1})} \|\bm{x_1}-\hat{\bm{x}}_1 \|_2 +  \E_{\bm{x_2}\sim p(\bm{X_2})} \|\bm{x_2}-\hat{\bm{x}}_2 \|_2 \\
        & + \ldots \E_{\bm{x_v}\sim p(\bm{X_v})} \|x_v-\hat{\bm{x}}_v \|_2\\
                 & =  \frac{1}{v}\sum_{i=1}^v  \E_{(\bm{x_1},\bm{x_i})\sim p(\bm{X_1}, \bm{X_i})} \|\bm{x_1}-F_1(G_i(\bm{x_i})) \|_2 \\
                 & +   \frac{1}{v}\sum_{i=1}^v \E_{\bm{x_i}\sim p(\bm{x_i})} \|\bm{x_i}-F_i(G_i(\bm{x_i})) \|_2
    \end{split}
\end{equation}
where the first summation is the loss of transforming the representation generated from the other views back to the first view, and the second term is the reconstruction error of transforming the representation back to its original view.
To leverage the label information, we propose to update the cross-entropy loss as follows:
\begin{equation}
    \begin{split}
        L_c^{(v)}  = E_{(\bm{x_1}...\bm{x_v})\sim p(\bm{X_1}...\bm{X_v})} \mathbb{H}(\bm{\bm{Y}}, \hat{Y})
    \end{split}
\end{equation}
where $\hat{Y}=C(z_s \oplus s_1 \oplus s_2\oplus...\oplus s_v)$ and $z_s= \frac{1}{v}\sum_{i=1}^{v}z_i$.

\section{Experimental Results}
In this section, we demonstrate the performance of our proposed algorithm \method ~in terms of the effectiveness by comparing it with state-of-the-art methods.

\subsection{Experiment Setup}
\textbf{Data sets and experiment setting:}
We mainly evaluate our proposed algorithm on the following data sets: one semi-synthetic data set based on WebKB\footnote{\url{http://www.cs.cmu.edu/~webkb/}}, one synthetic data set, four real world data sets, including Noisy MNIST~\cite{WangALB15}; XRMB~\cite{westbury1994x}, CelebA~\cite{liu2015deep} and  Caltech-UCSD Birds~\cite{WelinderEtal2010}. Table~\ref{table_stat} shows the statistics of these data sets. In the experiments, we set $t_1=2$, $t_2=2$, $t_3=3$, $\alpha=1$, $\beta=1$, the initial learning rate to be 0.03 with decay rate 0.96 if not specified, and the optimizer is momentum stochastic gradient descent.
To yield richer and more discriminative feature representation, we adjust the dimensionality of the representation denoted as $h$ based on the dimensionality of the input data, and the dimensionality $h$ is specified for each data set. The number of layers for the encoders, the decoders, and the discriminator can be adjusted to specific application scenarios. Since all the deep model baselines (e.g., Deep IB) are using fully connected layer networks, in the experiments, we set the encoders, the decoders, and the discriminator to be three-layer fully connected layer networks, and the classifier to be a two-layer fully connected layer network.

\begin{table*}
\centering
\begin{tabular}{|*{5}{c|}}
\hline \textbf{Data Sets}       & \textbf{\# of Training Samples} & \textbf{\# of Test Samples} & \textbf{\# of views}  & \textbf{\# of labels}  \\
\hline WebKB            & 6,626     & 1,656     &  2    & 7   \\
\hline Noisy MNIST      &  32,000   & 10,000    &  2    & 10  \\
\hline XRMB             &  75,000   & 15,000    &  2    & 15  \\
\hline CalebA           &  45,000   & 12,000    &  2    & 40  \\
\hline Caltech-UCSD Bird      &  950      & 238       &  2    & 40  \\
\hline
\end{tabular}
\caption{Statistics about five real-world data sets}
\vspace{-0.3cm}
\label{table_stat}
\end{table*}

\noindent \textbf{Reproducibility:}
All of the real-world data sets are publicly available. The code of our algorithms could be found via the link~\footnote{\url{https://github.com/Leo02016/ANTS}}. The experiments are performed on a Windows machine with 8GB GTX 1080 GPU.

\noindent \textbf{Comparison methods:}
In our experiments, we compare with the following methods: FCL, a four-layer fully-connected neural network trained with two concatenated views; Linear CCA~\cite{ChaudhuriKLS09}, linear transformations of two views; DCCA~\cite{andrew2013deep}, nonlinear transformations of two views by canonical correlation analysis; DCCAE~\cite{WangALB15}, deep canonical correlated auto-encoders for two views; Deep IB~\cite{wang2019deep}, deep information bottleneck for multi-view learning. Following~\cite{WangALB15}, we also use SVM to make the final prediction for CCA based methods, after finding the latent representation. \\

\begin{table}
\centering
\begin{tabular}{|*{5}{c|}}
\hline \multicolumn{3}{|c|}{WebKB}\\
\hline \textbf{Model}       & \textbf{Accuracy} & \textbf{F1 Score}  \\
\hline linear CCA    & 0.575 $\pm$ 0.012  &    0.522 $\pm$ 0.013  \\
\hline DCCA          & 0.617 $\pm$ 0.009  & 0.604 $\pm$ 0.015  \\
\hline DCCAE         & 0.652 $\pm$ 0.010  &    0.629 $\pm$ 0.010  \\
\hline FCL           & 0.680 $\pm$ 0.002  &    0.656 $\pm$ 0.003  \\
\hline Deep IB           & 0.650 $\pm$ 0.009  &    0.646 $\pm$ 0.009  \\
\hline \method      & \textbf{0.708 $\pm$ 0.004} & \textbf{0.686 $\pm$ 0.004}    \\
\hline \multicolumn{3}{|c|}{Synthetic Data} \\
\hline \textbf{Model}   & \textbf{Accuracy}  & \textbf{F1 Score}\\
\hline linear CCA    &  0.674    $\pm$ 0.009  & 0.672 $\pm$ 0.006 \\
\hline DCCA          &  0.709 $\pm$ 0.007  & 0.711 $\pm$ 0.009 \\
\hline DCCAE         &  0.778 $\pm$ 0.007  & 0.773 $\pm$ 0.012 \\ 
\hline FCL           &  0.897 $\pm$ 0.004  & 0.897 $\pm$ 0.004 \\
\hline Deep IB           &  0.854 $\pm$ 0.007  & 0.854 $\pm$ 0.007 \\
\hline \method  & \textbf{0.953 $\pm$ 0.005} &    \textbf{0.953 $\pm$ 0.005}    \\
\hline
\end{tabular}
\caption{Results on WebKB and synthetic data sets}
\vspace{-0.3cm}
\label{table1}
\end{table}

\begin{figure}
\begin{center}
\begin{tabular}{c}
\includegraphics[width=0.75\linewidth]{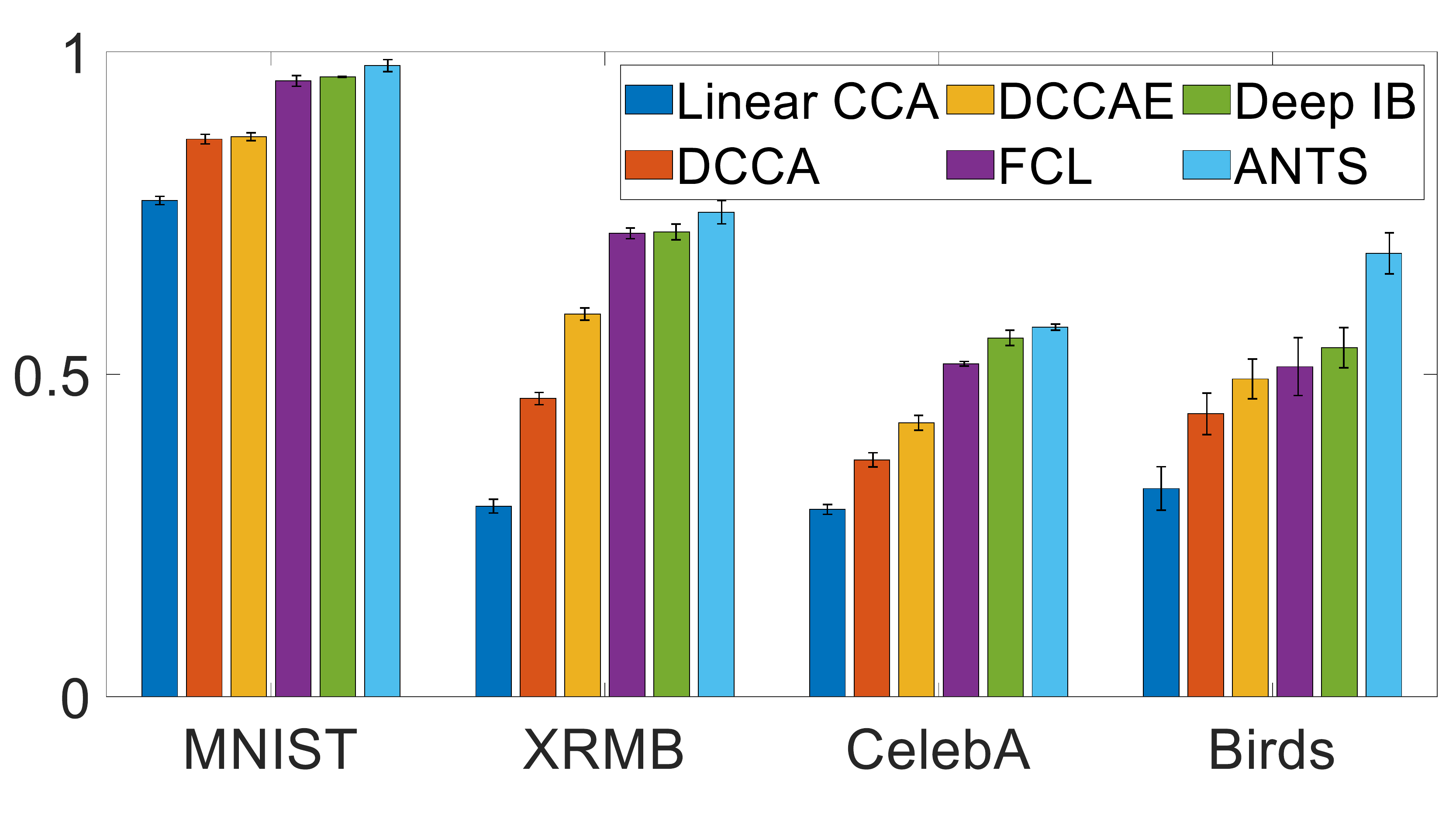} \\
(a) Accuracy \\
\includegraphics[width=0.75\linewidth]{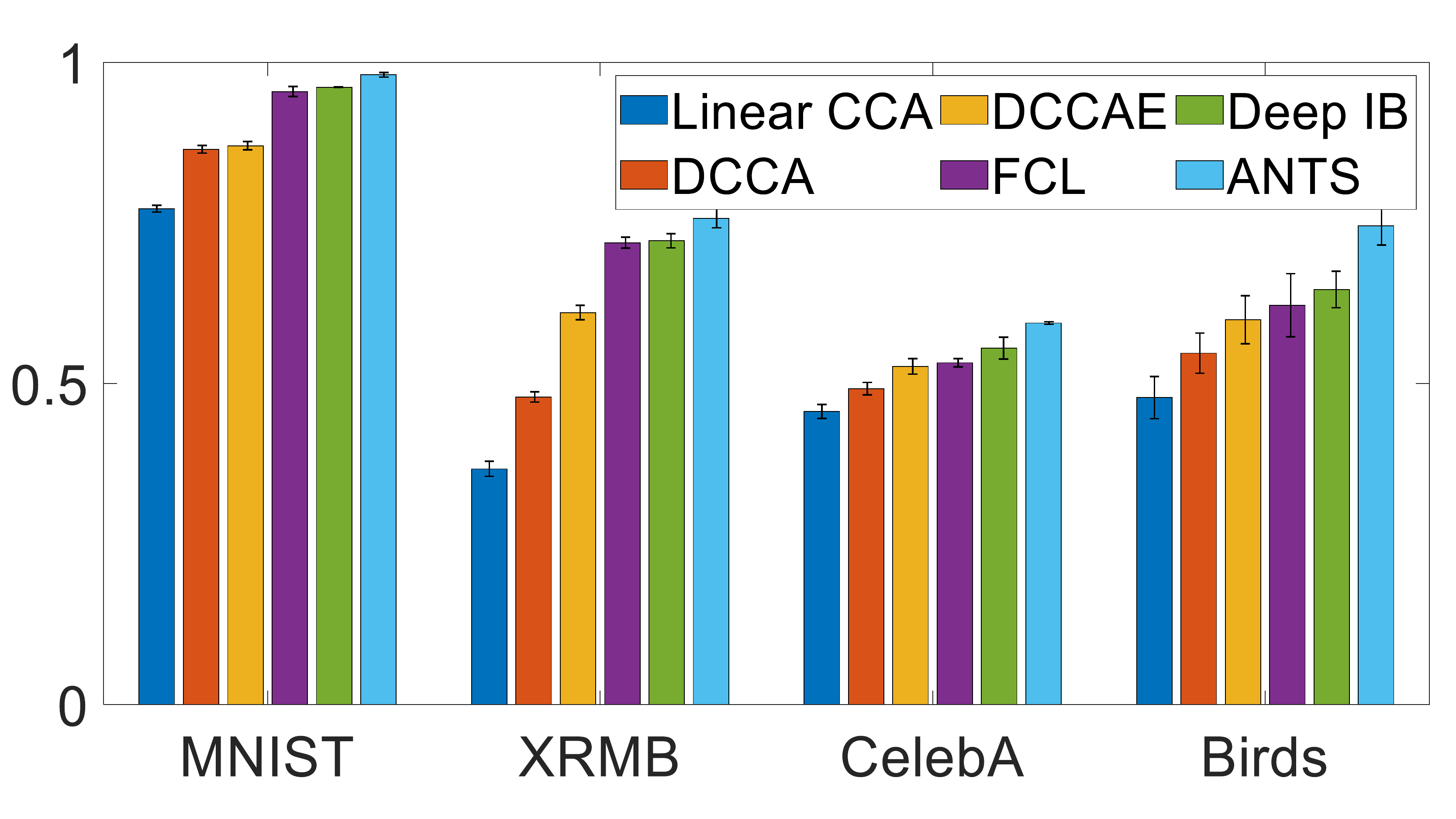}\\
(b) F1 Score \\
\end{tabular}
\end{center}
\caption{Performance comparison on Noisy MNIST, XRMB, CelebA and Caltech-UCSD Birds (Best viewed in color)}
\vspace{-0.3cm}
\label{fig1}
\end{figure}

\subsection{Semi-synthetic and Synthetic Data Set}
The goal of these two semi-synthetic/synthetic data sets is to show that CCA based methods cannot handle the noisy data as they ignore the label information, thus leading to a decrease in prediction performance. In our experiments, we evaluate the performance of our proposed method on one semi-synthetic and one synthetic data set.
 
The semi-synthetic data set is based on an online textual data set named WebKB, which consists of over 8000 web pages from multiple universities, and these web pages are manually classified into seven categories (labels), e.g., student, faculty, staff, course, project, etc. We extract TF-IDF representation by converting the content of the web pages to word vectors. The dimensionality of this representation is 10,111 after we remove the stop words. Then, we use two different non-linear mapping functions ($\sigmoid(\cdot)$ and $\tanh(\cdot)$) after normalization to construct the two views. Besides, we add random normal distribution noise  ($\mathcal{ N}(0, 0.2)$) to the constructed views and shuffle the order of the features. The final dimensionalities of the two views are both 10,611. For our proposed algorithm, we set the dimensionality of the representation to be 300, $k_1=k_2=50, d_1=d_2 = 233$ (we use the zero-padding to fill the last segment) and we use 80\% data as the training set and the remaining 20\% as the test set.

The synthetic data set is generated in the following way. First, we sample 10,000 data points with the feature dimensionality of 500 categorized into two balanced classes based on the method used in~\cite{guyon2003design}. Then, we use two different non-linear mapping functions ($\sigmoid(\cdot)$ and $\tanh(\cdot)$) after normalization to construct two views. Besides, we add random noise from two normal distributions ($\mathcal{ N}(0, 0.5)$ and $\mathcal{N}(1, 0.7)$) to the constructed views and shuffle the order of the features. 
The noise with the dimensionality of 50 is added to both views and the final dimensionality of both views are 550. For our proposed algorithm, we set the dimensionality of the representation to be 100, $k_1=k_2=11, d_1=d_2=50$ and we use 80\% data as the training set and the remaining 20\% as the test set.

The comparison results in terms of the accuracy and the F1 score are shown in Table~\ref{table1}. This table shows that our proposed model is better than the others in terms of both evaluation metrics. By observation, we find that for the synthetic data set, the accuracy of \method~ reaches 0.953 compared with 0.897 for the second-best baseline. This shows that our proposed method handles noisy data very well because the co-attention module helps users filter out the noisy features by putting a small weight on these noisy features, while the CCA based methods achieve suboptimal performance as they do not utilize the label information to derive the latent representation. For the semi-synthetic data set WebKB, our proposed method achieves the best performance compared with all other baselines.

\subsection{Real World Data Sets}
Next, we test the performance of our proposed method on four real-world data sets, including Noisy MNIST, XRMB, CelebA, and Caltech-UCSD Birds. The first data set is generated from MNIST data set~\cite{lecun1998gradient}. We rotate the original MNIST data set at angles randomly chosen from [$-\frac{\pi}{4},\frac{\pi}{4}$] to generate the first view and blur the images to create the second view. Both views have 784 features and 10 digits are considered as labels. Because some baselines are time-consuming, we repeat the experiments by randomly sampling 32,000 data points for training and another 10,000 data points for testing. For our proposed method, we set $k_1=k_2=7, d_1=d_2=112$. The second data set is from the Wisconsin X-ray Microbeam Database, which consists of two views. The first view is acoustic data with 273 features and the second view is articulatory data with 112 features. We repeat the experiments by randomly sampling 75,000 data points for training from the first 15 classes and another 15,000 data points for testing. For our proposed method, we set the dimensionality of the representation to be 50, $\alpha=2$, $\beta=1$, $k_1=k_2=5, d_1=55$, and $d_2=23$ (we use the zero-padding to fill the last segment). The third data set is composed of 202,599 images of celebrities, and 40 labeled facial attributes. We repeat the experiments by randomly sampling 45,000 data points for training and over 12000 images for testing. In this experiment, we choose 5 facial attributes related to hair as the labels, e.g., bald, black hair, straight hair, wavy hair, and wearing a hat. To get two views for this data set, we adopt the idea of creating two views in the Noisy MNIST data set by randomly rotating an image to create the first view and blurring the image to generate the second view. To reduce the run-time for the baselines, we first convert the RGB images to gray-scale images and then resize them to 100 $\times$ 100 to be fair to all competitors. Therefore, the final dimensionalities of the two views are both 10,000. For our proposed method, we set $k_1=k_2=50, d_1=d_2=200$. 

For the last data set, we sample 1188 images from the first 40 classes. To get the two views, we follow the similar procedure of generating Noisy MNIST~\cite{WangALB15}. We rotate each image at angles randomly chosen from [$-\frac{\pi}{4},\frac{\pi}{4}$] to generate the first view. The procedure for generating the second view consists of two steps. In the first step, we compute the mean $\mu$ and the standard deviation $\sigma$ of each image $x_2^j$, generate the random noise $\mathcal{N}(\mu, \sigma)$. In the second step, we sample $5\%$ index denoted as \textit{sampled\_index} of the image $x_2^j$ and add the noise to the image based on the following function, $x_2^j(sampled\_index) = x_2^j(sampled\_index) + 0.1 * \mathcal{N}(\mu, \sigma)$. Considering the high time complexity of some baselines (\eg, CCA, DCCAE, Deep IB) for high dimensional data like images, we fine-tune a pre-training model (\eg, VGG-19~\footnote{https://pytorch.org/docs/stable/torchvision/models.html}) on our dataset to extract the hidden representation of each image and then use the representation as the input for our experiments. The dimension of the hidden representation for both views is 1024. For our proposed method, we first apply the Mask-RCNN~\cite{he2017mask} to split each image into three segments, and then we also use the same pre-training model to get the representation for each segment of each image (if the image could be split into more than three segments, we pick the top two segments with the most pixels and keep the rest in the third segment; if the image could only be split into two segments, we randomly pick a segment as the third one). For this dataset, we set the learning rate to be 0.08 with decay rate 0.96, the batch size to be 3 (as we find a large batch size leads to the worse performance) and the number of epoch to be 150, $\alpha=2$, $\beta=1$ and we use 80\% data as the training set and the remaining 20\% as the test set.

For these four real-world data sets, the y-axis is accuracy in Figure~\ref{fig1} (a) and F1 score in Figure~\ref{fig1} (b), respectively. These figures show that our proposed model, \method, outperforms the others with both evaluation metrics in all three data sets. Notice that on the Noisy MNIST data set, most CCA based methods achieve an accuracy above 85\%, while \method~ even reaches the accuracy of 98.05\%, compared with the accuracy of 96.12\% for the second-best method (Deep IB). As we mentioned, the two views of the Noisy MNIST data set are either contaminated by the Gaussian noise or distorted by the random rotation but the experimental results on this data set demonstrate that our proposed method could handle the noise very well. 
In the XRMB data set, \method~ boosts the prediction performance by more than 3.5\% compared with the second-best baseline in terms of both accuracy and F1 score. 
In Caltech-UCSD Birds data set, \method~ outperforms the second-best algorithm by more than 8\%, which demonstrates the effectiveness of handling the noisy features for our proposed method. To show how the proposed method interprets the predictions, we conduct a case study in Section 5.1, which further shows why our proposed method outperforms state-of-the-art algorithms in this data set.

\subsection{Parameters Analysis}
In this section, we analyze the parameter sensitivity of our proposed~\method~ algorithm on Noisy MNIST data set, including $t_1$, $t_2$,  $t_3$, the dimensionality of the representation $h$ and the hyper-parameters $\alpha$ and $\beta$. In all experiments, we use 10,000 data point as the training data and 10,000 data point as the test data; we set the batch size to be 50, and the total iterations number to be 10,000 and the learning rate to be 0.03. In this parameters analysis, if not specified, we let $h=30$, $t_1=2$, $t_2=2$, $t_3=3$, $\alpha=1$ and $\beta=1$.

To evaluate $t_1$, $t_2$, and $t_3$, we increase the value of one parameter and fix the other three parameters to be 1, parameter $h$ to be 30. The x-axis in the Figure~\ref{parameter_k} (a) is the value of each parameter, e.g. $t_1=1,\ldots, 5$ and the y-axis is the classification error. Based on the observations, we find that the classification error increases as $t_1$  increases, and a large value of $t_1$ leads to a bad representation and high classification error. The classification error decreases as $t_2$ and $t_3$ increases. The model reaches its optimal state when $t_2=2$ and $t_3=3$ and does not change much when $t_2$ and $t_3$ increases. In conclusion, $t_1$ and $t_2$ should be set to an identical number, e.g.  $t_1=2$ and $t_2=2$, since we need to balance the loss among the discriminator and the generators. The suitable value for $t_3$ could be 3 based on the classification error shown in the Figure~\ref{parameter_k} (a). Similar to the Figure~\ref{parameter_k} (a), the x-axis of the Figure~\ref{parameter_k} (b) is the dimensionality of the latent representation and the y-axis is the classification error. Based on the classification error in Figure~\ref{parameter_k} (b), we observe that the dimensionality of representation does not influence the classification error too much, although the model achieves the slightly better performance when $h=30$ than other value of $h$. However, when the dimensionality of representation is less than 10, the classification error increases rapidly. One explanation for this is that the lower-dimensional representation (h<10) fails to capture all useful information from original data, thus leading to an increase in classification error. For the hyper-parameters $\alpha$ and $\beta$ that balance the cross reconstruction loss and the classification loss, we fix one parameter and increase the value of another one. In Figure~\ref{parameter_k} (c), the y-axis is the classification error and the x-axis is the value of these hyper-parameters. By observation, we find that when we increase the value of alpha or beta, the classification error also increases slightly. One possible explanation for this observation is that the latent representation is slightly influenced by the imbalance between the cross reconstruction loss and the classification loss.

\begin{figure}
\begin{center}
    \includegraphics[width=0.66\linewidth]{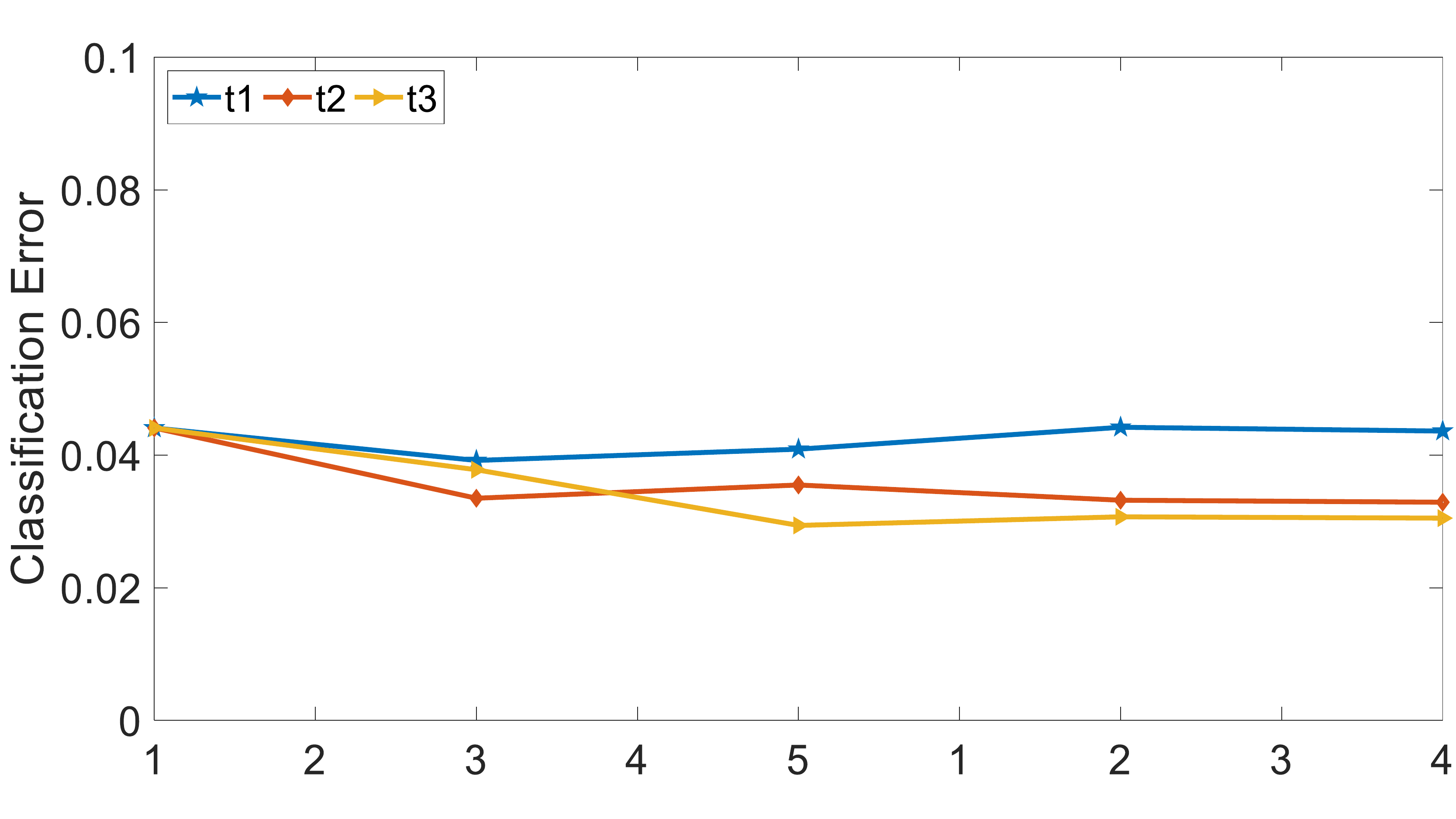} \\
    (a) Parameter $t_1$,$t_2$ and $t_3$  \\
    \includegraphics[width=0.66\linewidth]{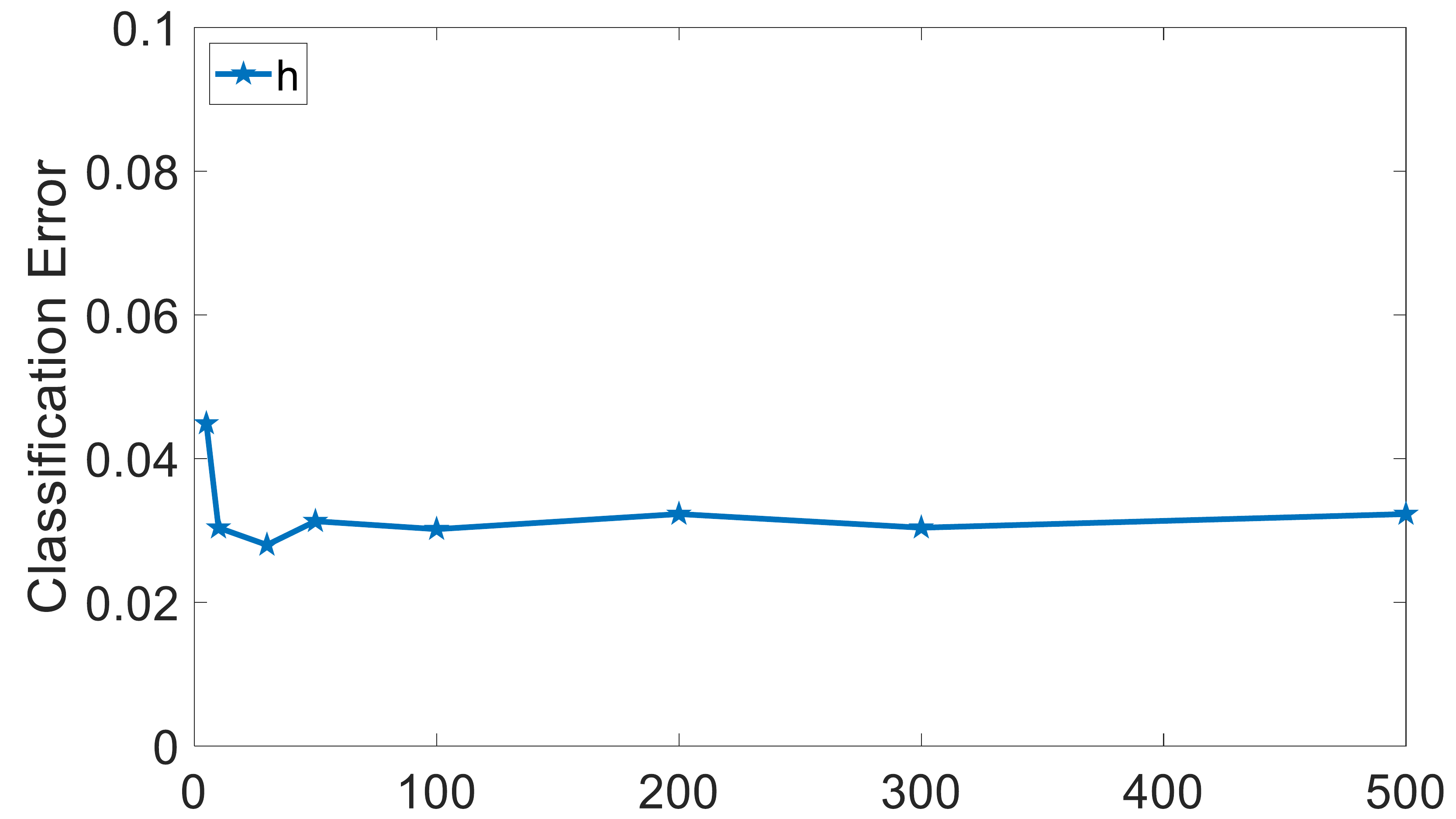} \\
     (b) Parameter $h$ \\
    \includegraphics[width=0.66\linewidth]{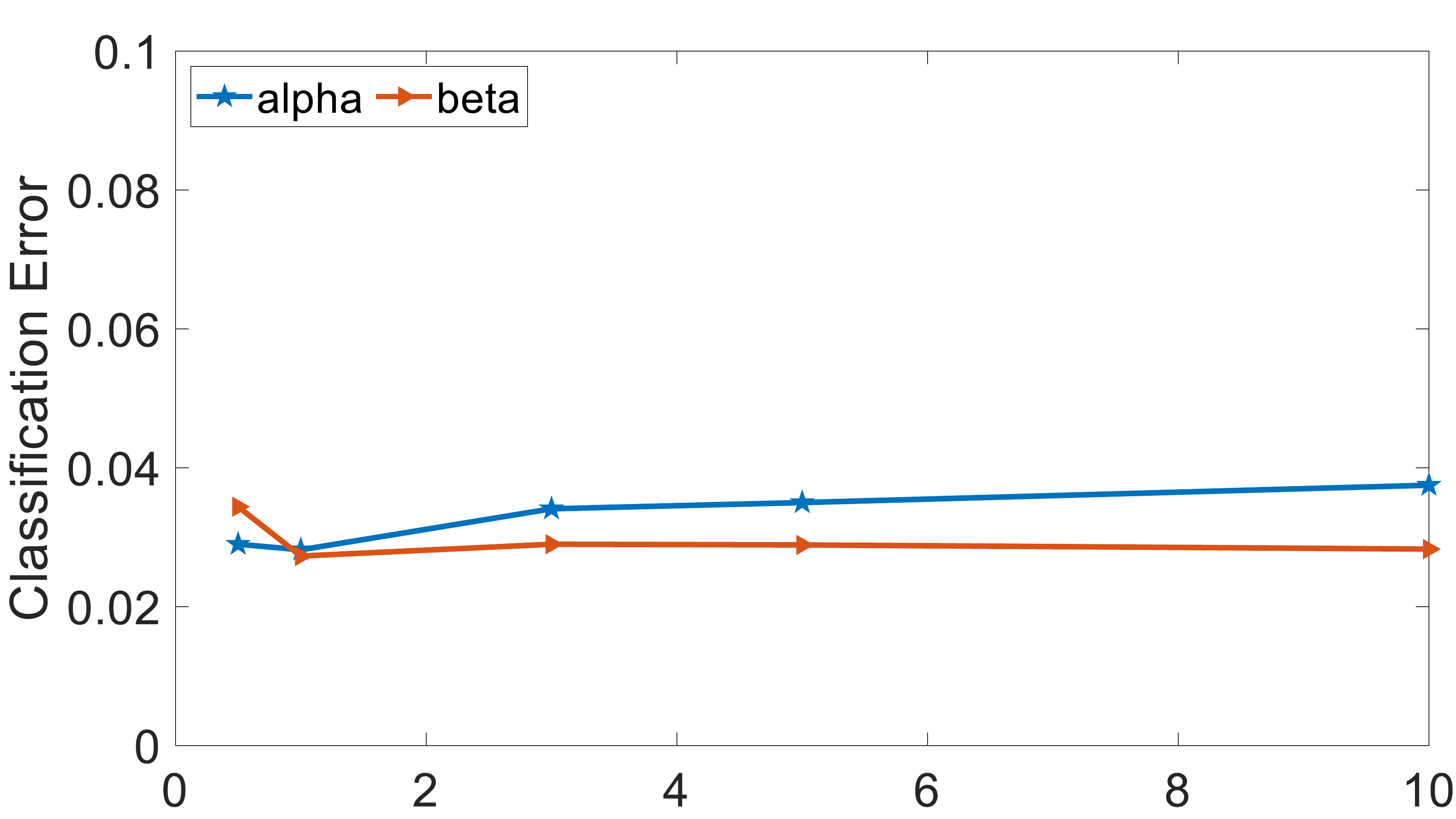} \\
    (c) Parameter $\alpha$ and $\beta$ \\
\end{center}
\caption{Parameter analysis (Best viewed in color)}
\vspace{-0.3cm}
\label{parameter_k}
\end{figure}
\section{Case Studies}
In this section, we present a case study on Caltech-UCSD Birds to show the interpretability of the proposed method in Subsection 5.1, and case studies on Noisy MNIST to demonstrate the effectiveness of adding complementary information in Subsection 5.2. Finally, We compare several different regularizers in Subsection 5.3.

\begin{figure*}[t]
\includegraphics[width=0.85\linewidth]{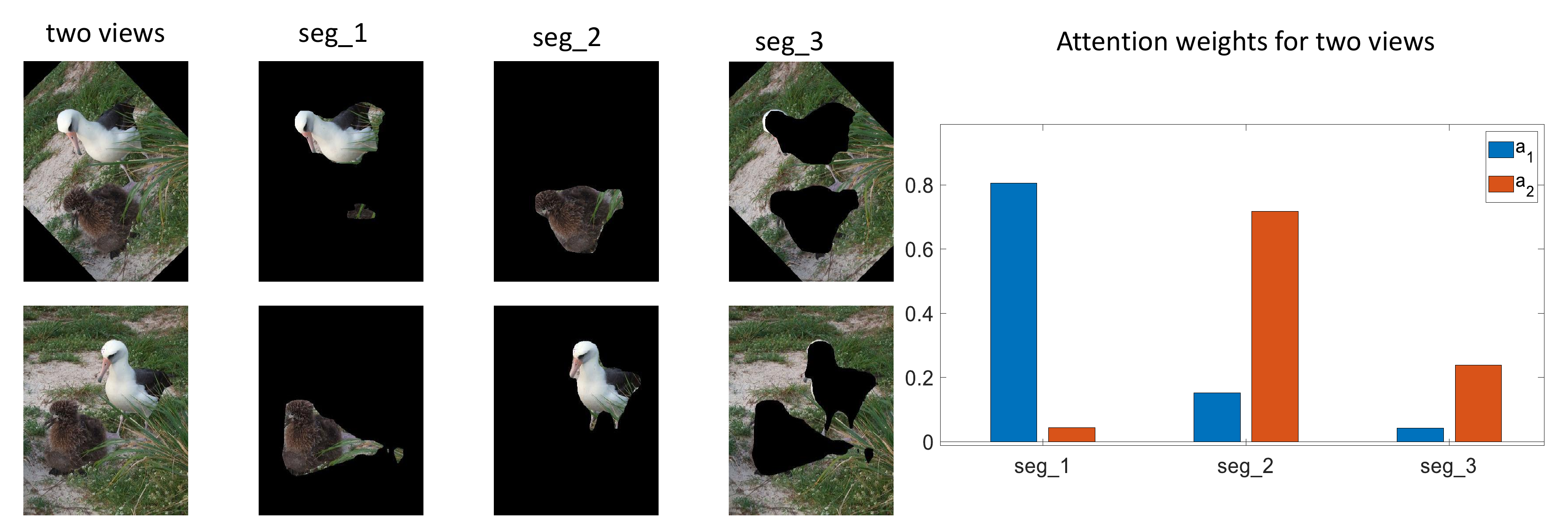}\\
\centering (a) Predicting Laysan Albatross\\
\includegraphics[width=0.85\linewidth]{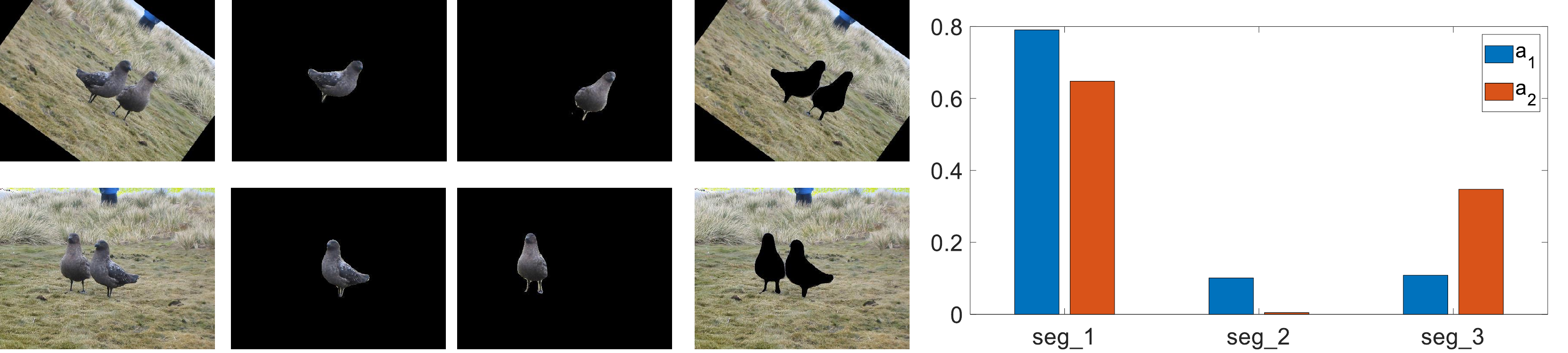}\\
\centering (b) Predicting Sooty Albatross\\
\includegraphics[width=0.85\linewidth]{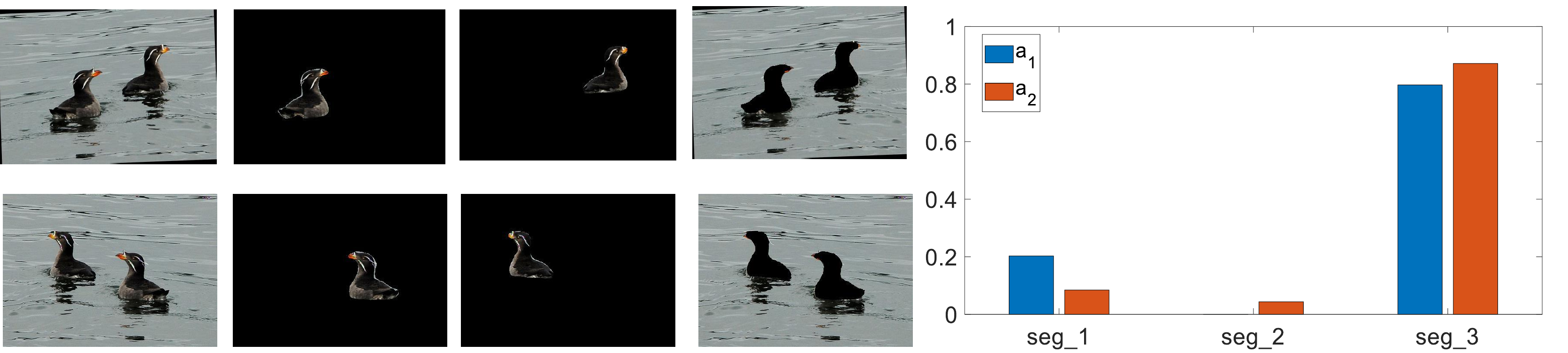}\\
\centering (c) Predicting Rhinoceros Auklet\\
\includegraphics[width=0.85\linewidth]{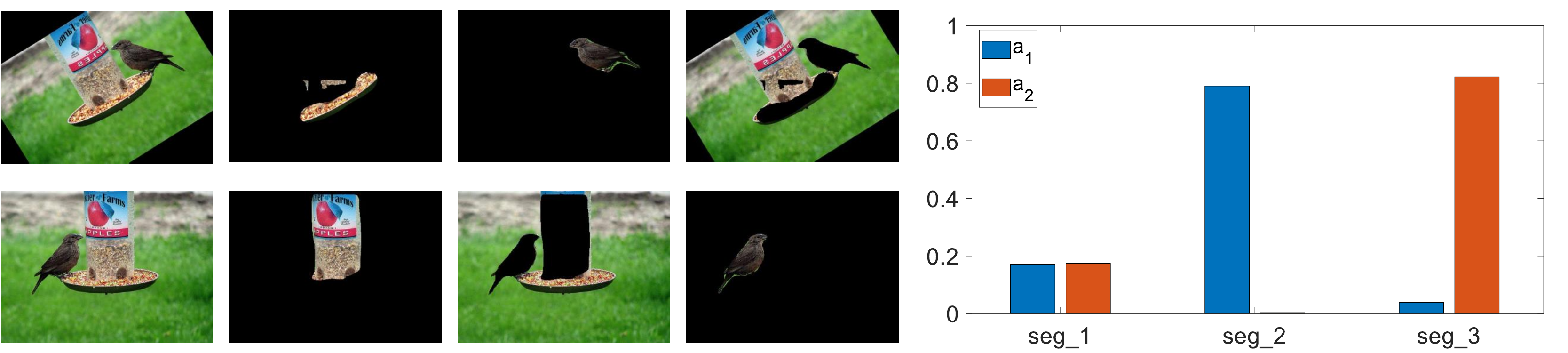}\\
\centering (d) Predicting Rusty Blackbird\\
\caption{A case study on the Caltech-UCSD Birds (Best viewed in color). In each sub-figure, the first row is the first view and the second row is the second view; seg\_1, seg\_2  and seg\_3 are three image segments for each view. The x-axis of the bar chart is the index of the segment, the y-axis is the attention weight, and the bars in blue (orange) are the attention weights for the first (second) view.}
\label{bird_case_study}
\end{figure*}

\subsection{Interpretation via Co-attention}
In many real-world applications, the end-user must understand the prediction results from predictive models.
Here we conduct a case study to interpret the prediction results and provide relevant clues for the end-users to investigate the vulnerability of the proposed algorithm in the noisy environment. We visualize the experimental results on the Caltech-UCSD Birds in the Figure~\ref{bird_case_study}, where the Figure~\ref{bird_case_study} (a) and (b) are selected from the training set and the Figure~\ref{bird_case_study} (c) and (d) are selected from the test set. The experimental setting is stated in Section 4.3. In each sub-figure, the first row is the first view and the second row is the second view; seg\_1, seg\_2  and seg\_3 are three image segments for each view. The attention wights of each image segment for two views, (\eg, $a_1$ and $a_2$ in Equation~\ref{eq:attention_weights}) are visualized in the bar charts on the right-hand side. The x-axis of the bar chart is the index of the segment, the y-axis is the attention weight, and the bars in blue (orange) are the attention weights for the first (second) view. 

Based on these four sub-figures, we have the following observations. First, we find that in Figure~~\ref{bird_case_study} (a) and (b), \method~ assigns larger weights to the important segments, compared with the bar chart in Figure~~\ref{bird_case_study} (a) and (b). This happens because (a) and (b) are selected from the training set, while (c) and (d) are from the test set. Second, in Figure~~\ref{bird_case_study} (a), there are two different types of birds, and the label of this image is "Laysan albatross", a large white seabird. \method~ assigns 0.80 to the seg\_1 of the first view and 0.72 to the seg\_2 of the second view, which suggests that \method~ successfully captures the most important features of this bird to make an correct prediction. Third, in Figure~\ref{bird_case_study} (c), \method~ makes an incorrect prediction because the model puts a large weight on the wrong features (the background of the image) instead of the features of the birds. With these observations, we could draw the following conclusions. First, our proposed model interprets the prediction results and shows the end-users why the model makes an correct or incorrect prediction by assigning the weights to different segments for each view. Second, our proposed methods is capable of handling both noises very well, as two views are generated in noisy environment.

\begin{figure}
\centering  
\includegraphics[width=0.75\linewidth]{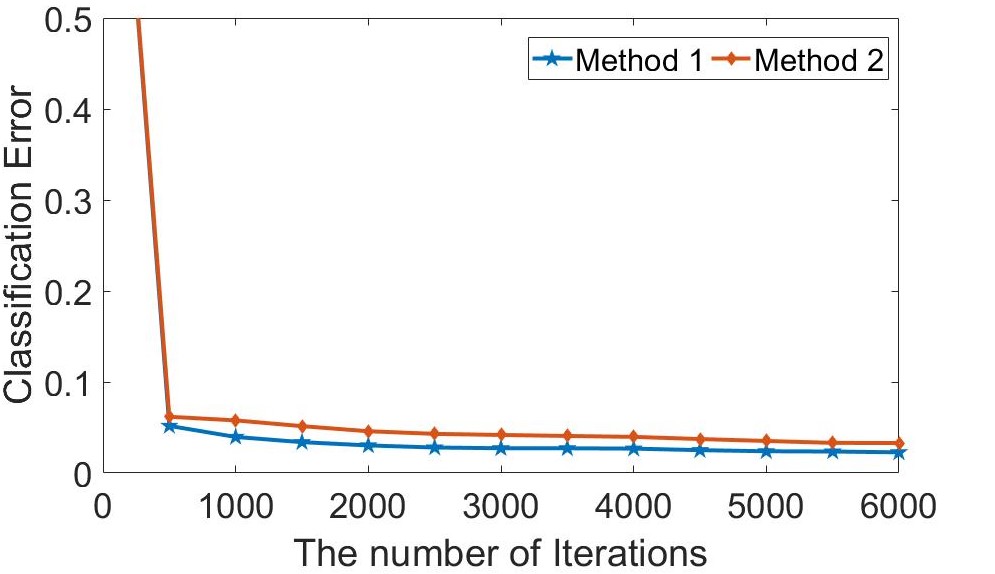}\\
(a) \\
\includegraphics[width=0.75\linewidth]{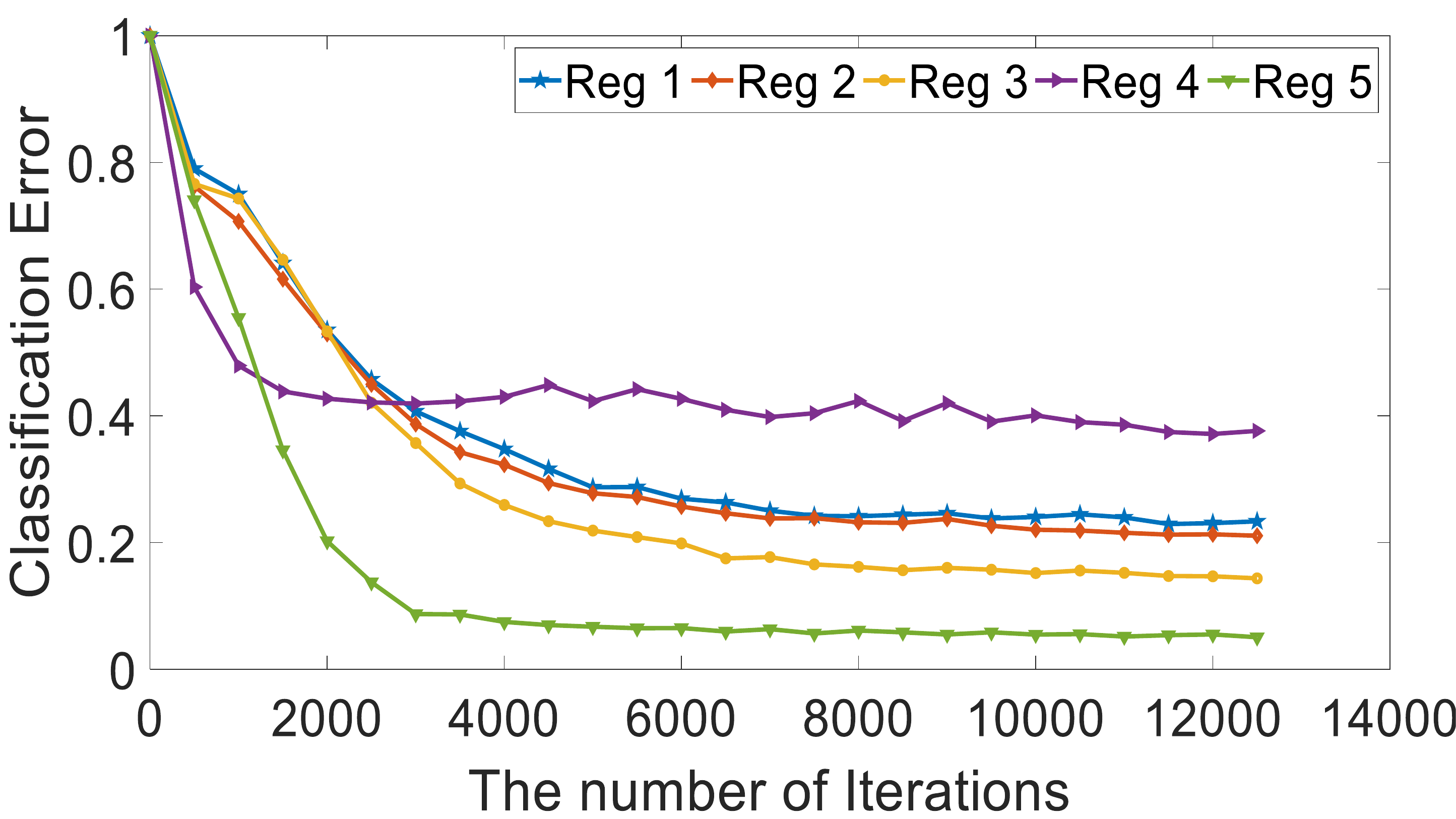}\\
(b)  \\
\caption{(a) Classification error vs the number of iterations on Noisy MNIST data set.
(b) Analysis of the quality of the representation generated by different regularizers. (Best viewed in color)}
\label{fig_case_study}
\end{figure}

\subsection{Ablation Study}
In the ablation study, we aim to demonstrate that the boost in the predictive performance is due to the view-specific information rather than the increased dimensionality of the hidden representation. Thus, we examine our hypothesis in the experiment on the Noisy MNIST data set based on the metric of the classification error. The methods for comparison are listed as follows:
\begin{itemize}
    \item \textit{Method 1}: the proposed method utilizing both the shared information and the view-specific information, namely, \method.
    \item \textit{Method 2}: the proposed method utilizing the shared information only but the shared information is copied three times.
\end{itemize}
Since \method~utilizes not only the shared information but also view-specific information, for a fair comparison, in \textit{Method 2}, the shared information is copied three times such that the dimensionality of the hidden layer in \textit{Method 2} is the same as that of \textit{Method 1}. In Figure~\ref{fig_case_study} (a), we observe that at the iteration of 500, the classification error of both methods decreases to 0.05. Notice that the classification error of \textit{Method 1} is less than 2.3\%,  while \textit{Method 2} drops to 4\% in the end, which demonstrates that leveraging complementary information could further boost the predictive performance.

\subsection{Comparison of View Reconstruction Regularizers}
In this case study, we demonstrate the effectiveness of the \textit{cross reconstruction loss} and the effectiveness of utilizing the label information by evaluating the quality of the generated representation with different regularizers on the Noisy MNIST data set.
The regularizers for comparison are listed as follows:
\begin{itemize}
    \item \textit{Reg 1}: Using $F_1(G_1(x_1))$ to reconstruct $x_1$ and $F_2(G_2(x_2))$ to reconstruct $x_2$ without utilizing the label information.
    \item \textit{Reg 2}: Using $F_1(G_2(x_2))$ to reconstruct $x_1$ and $F_2(G_1(x_1))$ to reconstruct $x_2$ without utilizing the label information.
    \item \textit{Reg 3}: Cross reconstruction loss without utilizing the label information (Equation~\ref{loss_3}).
    \item \textit{Reg 4}: The $L_1$ norm of the difference between two representations generated from two views $\|G_1(x_1) - G_2(x_2)\|_1$ without utilizing the label information.
    \item \textit{Reg 5}: Cross reconstruction loss utilizing the label information.
\end{itemize}

Because we only aim to find out how different regularizers imposed on the encoders influence the generation of the representation, we exclude the classifier during the training stage in Algorithm~\ref{alg1} for the first four regularizers except for \textit{Reg 5}. In other words, we do not utilize the label information to constrain the generation of the representation except for \textit{Reg 5}. For all the experiments, we randomly sample 32,000 examples from the Noisy MNIST data set as our training data, and 10,000 examples as the test data. For a fair comparison, we set $t_1$ to be 2, $t_2$ to be 2 for all five regularizers, $t_3$ to be 3 for the last regularizers, the batch size to be 64, the total iteration number $T$ to be 12,500, and the initial learning rate to be 0.05 with decay rate 0.96. Furthermore, we evaluate the quality of the representation based on the classification error at every 500 iterations. In Figure~\ref{fig_case_study} (b), the y-axis is the classification error evaluated on 10,000 test examples, and the x-axis is the number of iterations. It is easy to see that: when the label information is not utilized during the training stage, \textit{Reg 2} converges slightly faster than \textit{Reg 1}; compared with \textit{Reg 1}, \textit{Reg 2} and \textit{Reg 4}, the model with cross reconstruction loss (\textit{Reg 3}) converges at the faster rate and the quality of the representation generated by \textit{Reg 3} is also better than the representation generated by \textit{Reg 1} and \textit{Reg 2}; when the label information is utilized, the quality of the generated representation is further enhanced, and the classification error rate drops rapidly (see the comparison between \textit{Reg 3} and \textit{Reg 5} in Figure~\ref{fig_case_study} (b)). We also observe that the classification error for the model with $L_1$ norm regularizer (\textit{Reg 4}) decreases to a local minimal error rate of 0.4195 after 3000 iterations and starts to vibrate. The main reason is that there are no constraints imposed to restrict the distribution of $P_{g_1}$ and $P_{g_2}$ and thus, the model with $L_1$ norm regularizer yields the trivial representation.
Based on the results shown in the Figure~\ref{fig_case_study} (b), we can conclude that our model with cross reconstruction loss has a faster convergence rate and tend to find better representation than other regularizers. This demonstrates that cross reconstruction loss indeed enforces the two encoders $G_1(\cdot)$ and $G_2(\cdot)$ to find the shared representation at a faster pace, as mentioned in Subsection 3.2. When the label information is leveraged during the training stage, the convergence rate and the quality of the representation are greatly improved.
\section{Conclusion} 
In this paper, we propose \method\ - a deep adversarial co-attention model for multi-view subspace learning. We extract and integrate both the shared information and complementary information to obtain a more comprehensive representation. By imposing the cross reconstruction loss and incorporating a classifier into the proposed framework, we further enhance the quality of this representation. We also extend our method to accommodate the more generic scenario with multiple views. We compare our proposed method with state-of-the-art techniques on synthetic, semi-synthetic, and real-world data sets to demonstrate that our method leads to significant improvements in the performance. Finally, case studies show how the proposed method interprets the predictive results on an image data set, the importance of leveraging complementary information, and the effectiveness of the regularizer imposed on view reconstruction.
\section{acknowledgement}
This work is supported by National Science Foundation under Award No. IIS-1947203 and IIS-2002540, and IBM-ILLINOIS Center for Cognitive Computing Systems Research (C3SR) - a research collaboration as part of the IBM AI Horizons Network. The views and conclusions are those of the authors and should not be interpreted as representing the official policies of the funding agencies or the government.
\bibliographystyle{abbrv}
\bibliography{reference}
\end{document}